\documentclass{article}
\PassOptionsToPackage{numbers, compress}{natbib}
\usepackage[preprint]{neurips_2025}

\usepackage[utf8]{inputenc} 
\usepackage[T1]{fontenc}    
\usepackage{url}            
\usepackage{amssymb,amsfonts,amsmath,amsthm} 
\usepackage{enumerate,enumitem,tikz,graphicx,mathrsfs,eucal,verbatim, bbm, derivative}
\usepackage{tikz, tikz-cd}
\usepackage{svg}
\usepackage{wasysym}
\usepackage{rotating}
\usepackage{hyperref}     
\usepackage{subcaption}
\usepackage{booktabs}
\usepackage{stmaryrd}
\usepackage[normalem]{ulem}

\theoremstyle{plain}
\newtheorem{thm}{Theorem}[section]

\newtheorem{conjecture}[thm]{Conjecture}

\newtheorem{exmp}[thm]{Example}
\newtheorem{prop}[thm]{Proposition}
\newtheorem{lemm}[thm]{Lemma}
\newtheorem{cor}[thm]{Corollary}
\theoremstyle{definition}
\newtheorem{defn}{Definition}
\theoremstyle{remark}
\newtheorem{rmk}{Remark}[section]

\newcommand{\symm}{\textnormal{Sym}}
\newcommand{\symProd}{\textnormal{Sym}}

\let\svthefootnote\thefootnote
\newcommand\freefootnote[1]{%
  \let\thefootnote\relax%
  \footnotetext{\hspace{-1.5em}#1}%
  \let\thefootnote\svthefootnote%
}

\title{Singular Learning and Occam's Razor \\ in Deep Monomial Networks}

\author{
Kathl\'en Kohn \\
KTH Stockholm \\
Digital Futures \\
\texttt{kathlen@kth.se}
\And
Giovanni Luca Marchetti \\
KTH Stockholm \\
\texttt{glma@kth.se} 
\And
 Farhan Shabir \\
No Affiliation\\
\texttt{fshabir@emory.edu}
\And
Vahid Shahverdi \\
Umeå University \\
\texttt{vahid.shahverdi@umu.se}
\And
Weisheng Wang\\
Utrecht University\\
\texttt{w.wang2@uu.nl}
}

\begin{document}

\maketitle

\begin{abstract}
In the optimization of neural networks, gradient dynamics are influenced by critical points that arise from the model's architecture. These critical points occur where the Jacobian of the model's parametrization is rank-deficient, and are the most pronounced singularities studied in Singular Learning Theory. We investigate such points in deep fully-connected networks with monomial activations  via tools from polynomial algebra such as Mason's Theorem. We show that, for sufficiently large activation degree, criticality occurs precisely at subnetworks, i.e., at parameter configurations where some neurons are inactive or redundant. This offers a mathematical perspective on the implicit bias in deep neural networks, explaining the tendency of these models to converge toward simpler functions. 
\end{abstract}


\section{Introduction}\label{sec:intro}
Deep neural networks are typically trained by following (a discretization of) the gradient flow of their objective. The equilibria of this dynamics are the \emph{critical points} of the objective, i.e., parameters where the gradient vanishes. These can be local optima (maxima and minima), or possibly-degenerate saddles. Interestingly, some critical points
are induced by the model and its architecture \cite{trager2019pure}, rather than depending directly on the given data or objective. They arise when the \emph{parametrization} of the model exhibits a Jacobian  that is rank-deficient. At these critical parameters, the Riemannian metric induced by the model -- typically, the Fisher Information Metric -- drops rank. Based on this, such parameters are also referred to as `singularities', and are the central focus of Singular Learning Theory (SLT) \cite{watanabe2009algebraic}. According to SLT, critical points of the parametrization play a role of implicit biases in learning, 
from both dynamical and Bayesian perspective \cite{wei2022deep}. 

In this work, we study the critical points of the parametrization of a certain class of neural networks. Namely, we consider deep fully-connected networks with a monomial activation function -- a  setting introduced in \cite{kileel2019expressive}, and later developed in several works \cite{finkel2024activation, usevich2025identifiability, kubjas2024geometry, massarenti2025alexander,alexandr2026algebraic,dao2026minimal}. For these models, the end-to-end function is a (homogeneous) multivariate polynomial. This makes them particularly amenable to theoretical analysis, since it is possible to apply  powerful tools of polynomial algebra. In fact, this work fits into a broad research program aimed at bridging algebra and algebraic geometry with deep learning -- a field recently termed as Neuroalgebraic Geometry \cite{marchetti2025invitationneuroalgebraicgeometry}. 

As a main result, we prove that, for sufficiently large activation degree, the  critical points of the parametrization are precisely \emph{subnetworks}, i.e., parameters of the network where some neurons are redundant or inactive. Equivalently, rank deficiency occurs exactly at parameter configurations where some hidden neuron can be removed without changing the realized function. 
In light of the relation between critical points and implicit biases, this can be interpreted as providing a mathematical perspective on the tendency of deep neural networks to converge, during training, towards `simple' functions -- more specifically, towards configurations where a number of neurons can be pruned. This tendency is a crucial behavior, suggesting that deep neural networks implement, implicitly, a form of the Occam's Razor.

Our result complements the main findings of \cite{finkel2024activation}, where it is shown that subnetworks are the only non-identifiable parameters\footnote{This statement is stronger than the original phrasing in \cite{finkel2024activation}, but follows from the same proof -- see Section \ref{sec:backident}.}, meaning that they admit larger symmetries than expected. Our proof strategy, however, differs from \cite{finkel2024activation}. While the latter relies on a polynomial version of Fermat's Last Theorem due to Newman-Slater \cite{newman1979waring}, we leverage on a stronger tool: Mason's polynomial version of the $abc$-conjecture from number theory. By combining Mason's Theorem with a subtle divisibility analysis of the polynomials defined by the network, we are able to describe the kernel of the Jacobian of the parametrization.

\section{Related Work}
\label{sec:relwork}

\paragraph{Algebraic geometry of deep learning.} As anticipated, a line of research in theoretical deep learning -- recently termed 
Neuroalgebraic Geometry \citep{marchetti2025invitationneuroalgebraicgeometry} -- explores the study of polynomial neural networks through the lens of algebraic geometry. Motivated by the fact that polynomials can approximate arbitrary continuous functions, the goal of Neuroalgebraic Geometry is to approach fundamental  problems in machine learning via tools from algebra and geometry. One such problem is identifiability, which represents a core focus of neuroalgebraic geometry. Several models have been  analyzed in the literature, ranging from deep fully-connected networks with linear \citep{trager2019pure, marchetti2025critical}, monomial \citep{kileel2019expressive, finkel2024activation, kubjas2024geometry, marchetti2024harmonics, massarenti2025alexander}, and general polynomial \citep{shahverdi2026learning, shahverdi2026identifiable} activations, to deep convolutional  networks with linear  \citep{kohn2022geometry, kohn2024function, shahverdi2024algebraic}, monomial  \citep{shahverdi2024geometryoptimizationpolynomialconvolutional, hendi2026geometry} and polynomial activations \cite{shahverdi2026learning}, to (un-normalized) attention-based networks \citep{henry2024geometrylightningselfattentionidentifiability}. Our work naturally fits into this line of research, since we consider fully-connected networks with monomial activations of large degree. We contribute by describing the critical points of the parametrization -- one of the basic questions in Neuroalgebraic Geometry, that has been answered only for linear fully-connected \cite{trager2019pure} and polynomial convolutional \cite{kohn2024function, shahverdi2024geometryoptimizationpolynomialconvolutional} networks.

\paragraph{Singular learning theory (SLT).} SLT is a discipline in theoretical machine learning focused on \emph{singularities}\footnote{The notion of singularity in SLT does not fully coincide with the one in algebraic geometry -- see \cite{shahverdi2026learning} for details.} -- points where the Fisher Information Metric induced in parameter space by a given model degenerates \citep{watanabe2009algebraic, watanabe2007almost, amari2003learning, wei2022deep}. According to SLT, singularities induce biases in the learning process. 
This has been analyzed, mathematically, both from a dynamical \cite{wei2008dynamics} and probabilistic \cite{watanabe2013widely} perspective. Now, singularities are closely related to critical points of the parametrization -- we expand on this in Section \ref{sec:relslt}. Thus, our work contributes to the SLT picture by relating singularities to subnetworks (for deep monomial networks).
As mentioned in Section \ref{sec:intro}, this can be interpreted as a mathematical justification of the implicit simplicity bias of contemporary models.

\paragraph{Implicit sparsity bias.} 
The tendency of deep neural networks to converge to sparse weights is well-established. Empirically, this phenomenon has been popularized by the celebrated `lottery ticket hypothesis' \cite{frankle2018lottery}. Theoretically,
a line of research \citep{woodworth2020kernel, nacson2022implicit,pesme2021implicit, andriushchenko2023sgd} has shown that, for deep diagonal linear networks, the training dynamics, when initialized around the origin, implicitly penalizes the $\ell_1$ norm of the weights, inducing sparsity in the representation. A similar bias towards low-rank solutions has been shown for deep linear networks trained with a regularized loss \citep{kunin2019loss, ziyin2022exact, wang2023implicit}. Closely related to our work, these sparsity biases have been reformulated in terms of subnetworks \citep{chen2023stochastic}, and even analyzed dynamically in simple scenarios \cite{kunin2026alternating, marchetti2026sequential}. While our analysis is not dynamical, our work contributes to sedimenting the fundamental role of subnetworks in the learning process, for an interesting class of models (deep polynomial networks).

\section{Background}
In this section, we overview the basic notions and results around deep polynomial neural networks.

\subsection{Deep Polynomial Networks and their Neuromanifolds}

Fix a function $\sigma \colon \mathbb{R} \rightarrow \mathbb{R}$, a sequence of $L > 1$ positive integers $d_0, \ldots, d_L$, and, for every $i=1, \ldots, L$, a matrix $W_i \in \mathbb{R}^{d_i \times d_{i-1}}$.  
\begin{defn}
A \emph{Multi-Layer Perceptron} (MLP) with architecture $\mathbf{d} = (d_0, \ldots, d_L)$, activation function $\sigma$ and weights $\mathbf{W} = (W_1, \ldots, W_{L})$ is the map $f_\mathbf{W} \colon \mathbb{R}^{d_0} \rightarrow \mathbb{R}^{d_L}$ given by the composition:
\begin{equation}\label{eq:mlpdef}
f_\mathbf{W} = W_{L} \circ \sigma \circ \cdots \circ \sigma \circ W_1,
\end{equation}
where $\sigma$ is applied coordinate-wise.
\end{defn}

We now introduce the function spaces parametrized by neural networks. Let $\mathcal{W} = \bigoplus_{i=1}^L \mathbb{R}^{d_i \times d_{i-1}}$ be the parameter space of an MLP, and  $\varphi \colon \mathcal{W} \ni \mathbf{W} \mapsto f_\mathbf{W}$ be its parametrization map. 

\begin{defn}
The \emph{neuromanifold} of an MLP with architecture $\mathbf{d}$ and activation function $\sigma$ is the image of the parametrization $\varphi$, i.e.,
\begin{equation}
\mathcal{M}_{\mathbf{d}, \sigma} = \left\{f_\mathbf{W} \ | \  \mathbf{W} \in  \mathcal{W} \right\}.
\end{equation}
\end{defn}
When $\sigma$ is a polynomial, the Tarski-Seidenberg Theorem implies that $\mathcal{M}_{\mathbf{d}, \sigma}$ is a \emph{(semi-)algebraic variety} -- the central object of real algebraic geometry. As discussed in Section \ref{sec:relwork}, this is the starting point of Neuroalgebraic Geometry. In this work, we focus on \emph{monomial activations}, meaning that $\sigma(z) = z^r$ for some $r \in \mathbb{Z}_{\geq 1}$, whose neuromanifold we denote by $\mathcal{M}_{\mathbf{d}, r}$. The study of this type of networks has been initiated by \cite{kileel2019expressive}, and has since become a cornerstone of Neuroalgebraic Geometry -- see Section \ref{sec:relwork}. In this case, $f_\mathbf{W}$ is a homogeneous polynomial of degree $r^{L-1}$. We denote by $\symm_{m}(d, d')$ the space of homogeneous polynomials of degree $m$ with $d$ inputs and $d'$ outputs. We then have the inclusion $\mathcal{M}_{\mathbf{d}, r} \subseteq \symm_{r^{L-1}}(d_0, d_L)$. Thus, the neuromanifold is contained in a finite-dimensional \emph{ambient space} $\mathcal{V}:= \symm_{r^{L-1}}(d_0, d_L)$.

\subsection{Subnetworks and Identifiability}
\label{sec:backident}
We now introduce the notion of a subnetwork of an MLP, which will be central in what follows. 
\begin{defn}
\label{def:subnet}
Let $0 < j < L$ and $1 \leq i \leq d_j$. The $i$-th neuron at the $j$-th layer is: 
\begin{itemize}
  \item \emph{inactive}, if its incoming or outcoming weights vanish, i.e., $W_j[i, :] =0$ or $W_{j+1}[:, i] =0$,
  \item \emph{redundant}, if $W_j[i, :]$ is non-vanishing, but proportional to $W_j[k, :]$ for some $1 \leq k \leq d_j$, $k \neq i$.  
\end{itemize}
The weights $\mathbf{W}$ are a \emph{subnetwork} if they admit at least one inactive
or redundant neuron. 
\end{defn}   
If $\mathbf{W}$ is a subnetwork, it is possible to remove inactive and redundant neurons, obtaining an MLP with a smaller architecture that defines the same function
$f_\mathbf{W}$. This is achieved by simply removing the $i$-th row of $W_j$ and the $i$-th column of $W_{j+1}$, thereby decreasing $d_j$ by $1$. For inactive neurons, this works on the spot, since they do not participate in the computation of the function, due to the vanishing of the corresponding weights. For redundant ones, the removal needs to be compensated by the relevant columns of $W_{j+1}$. Namely, if $W_j[i, :] = \lambda W_j[k, :]$, then $W_{j+1}[:,k]$ must be replaced by $W_{j+1}[:,k] + \lambda^{r} W_{j+1}[:, i]$. By iterating these procedures, all the inactive/redundant neurons can be removed. Independently of the order of the removal, this leads to an MLP with non-subnetwork weights $\mathbf{W}'$ and architecture $(d_0, d_1', \ldots, d_{L-1}', d_L)$, $d_i' \leq d_i$, such that $f_{\mathbf{W}'} = f_\mathbf{W}$. 

We illustrate subnetworks and the removal procedure in the following small example.
\begin{exmp}
    Consider the architecture $\mathbf{d}=(2,2,1)$ with activation $\sigma(z)=z^2$, and weights $\mathbf{W}=(W_1,W_2)$. The end-to-end function is
\begin{equation}
    f_{\mathbf{W}}(x) = W_2[1,1](W_1[1,:]x)^2 + W_2[1,2](W_1[2,:]x)^2.
\end{equation}
If the second hidden neuron is inactive, meaning that $W_1[2,:]=0$ or $W_2[1,2]=0$, then its contribution to $f_{\mathbf{W}}$ vanishes. Hence, we can remove the second row of $W_1$ and the second column of $W_2$, obtaining a network with architecture $\mathbf{d}'=(2,1,1)$ and the same end-to-end function.

If instead a hidden neuron is redundant, then $W_1[2,:]=\lambda W_1[1,:]$ for some scalar $\lambda$. In this case, the second term in $f_{\mathbf{W}}$ becomes $\lambda^2W_2[1,2](W_1[1,:]x)^2$, and can be absorbed into the first neuron. Thus, after removing the second hidden neuron, we can replace $W_2[1,1]$ by $W_2[1,1]+\lambda^2W_2[1,2]$, and the resulting smaller network realizes the same function.
\end{exmp}

We now review the main result from \cite{finkel2024activation}, addressing a conjecture from \cite{kileel2019expressive}. The result concerns the parameter identifiability problem for MLPs with a monomial activation of large degree, by describing the generic fibers\footnote{For a fixed function $f_{\mathbf{W}}$, the fiber of the parametrization over $f_{\mathbf{W}}$ is the set $\varphi^{-1}(f_{\mathbf{W}})=\{\mathbf{V}\in \mathcal{W} \mid f_{\mathbf{V}}=f_{\mathbf{W}}\}$. Here, generic means that the statement holds for parameters $\mathbf{W}$ outside a proper algebraic subset of $\mathcal{W}$.} of the parametrization map $\varphi$. Specifically, it is established that, for large degree $r \gg 0$ and generic weights $\mathbf{W}$, the function $f_\mathbf{W}$ determines $\mathbf{W}$ up to permutations and rescalings of neurons. The proof from \cite{finkel2024activation} actually shows that this is true exactly when $\mathbf{W}$ is not a subnetwork. For the sake of completeness, we discuss the result in this stronger form in Section \ref{sec:app_ident}. As a consequence of this result, by applying the fiber-dimension theorem from algebraic geometry, we deduce that the dimension of the neuromanifold is:
\begin{equation}
    \dim \mathcal{M}_{\mathbf{d}, r} = \dim \mathcal{W} - \dim \varphi^{-1}(f_{\mathbf{W}}) = d_L + \sum_{i=1}^L d_i(d_{i-1}-1),
\end{equation}
where $\mathbf{W}$ is not a subnetwork.

Note that if $\mathbf{W}$ is a subnetwork, then the fiber $\varphi^{-1}(f_\mathbf{W})$ is larger than the generic one. Since inactive neurons can be removed without affecting $f_\mathbf{W}$, their non-vanishing incoming / outcoming weights can be varied freely (and similarly for redundant neurons). More precisely, if some neuron at the $j$-th layer is redundant or inactive, then $\dim \varphi^{-1} (f_\mathbf{W})$ increases by at least $\min \{ d_{j-1}, d_{j+1} \} - 1$.       

\section{Motivation: Criticalities, and their Role in Learning}
In this section, we introduce the notion of criticality for the parametrization map $\varphi$, and discuss its role in learning, including the relation to SLT. 

Denote by $\textnormal{J}_\mathbf{W}\varphi$ the Jacobian of $\varphi$ at $\mathbf{W} \in \mathcal{W}$, which we interpret as a linear map $\textnormal{J}_\mathbf{W}\varphi \colon \mathcal{W} \rightarrow  \mathcal{V}$. 
\begin{defn}
    The weights $\mathbf{W} \in \mathcal{W}$ are a \emph{regular point} for $\varphi$ if the Jacobian attains its maximal possible rank, i.e., $\textnormal{rk}\; \textnormal{J}_\mathbf{W}\varphi = \textnormal{dim} \mathcal{M}_{\mathbf{d}, r}$. Otherwise, $\mathbf{W}$ is a \emph{critical point}.
\end{defn}
Sard's Theorem implies for the algebraic map $\varphi$ that it is regular almost everywhere in its domain $\mathcal{W}$. However, critical points are, perhaps, the most interesting ones. As mentioned in Section \ref{sec:intro}, the model is typically trained via the gradient flow of an objective. We think of the latter as a (data-dependent) function $\mathcal{L} \colon \mathcal{V} \rightarrow \mathbb{R}$. Then the model follows the gradient:
\begin{equation}
\label{eq:orthojac}
    \nabla_{\mathbf{W}}( \mathcal{L} \circ \varphi )^\top = \nabla_{f_\mathbf{W}} \mathcal{L}^\top  \cdot  \textnormal{J}_\mathbf{W}\varphi. 
\end{equation}
Now, given weights $\mathbf{V} \in \mathcal{W}$ that are annihilated by the Jacobian, i.e. $\mathbf{V} \in \ker \textnormal{J}_\mathbf{W}\varphi$, \eqref{eq:orthojac} immediately implies that  $\nabla_{\mathbf{W}}( \mathcal{L} \circ \varphi )^\top \mathbf{V} = 0$. In other words, the gradient of the loss function is orthogonal to the kernel of the Jacobian of the parametrization. This means that, locally, the training dynamics has no component along the kernel. 
Since critical points have larger kernel, the dynamics in their proximity exhibits larger constraints, determined by the Jacobian. This gives an indication, from a dynamical perspective, of an \emph{implicit bias} induced by critical points of the parametrization. As we shall see in the following, SLT builds upon this perspective, providing a comprehensive theory around closely related notions and phenomena. 

\begin{rmk}\label{rmk:largeFiberImpliesCritical}
  Critical points are closely connected to points that result in fibers  larger than the generic one, as discussed in  Section \ref{sec:backident}.
  More concretely, if the dimension of $\varphi^{-1}(\varphi(\mathbf{W}))$ is larger than the dimension of the generic fiber of $\varphi$, then $\mathbf{W}$ must be a critical point of $\varphi$
     because the tangent space of the fiber lies in the kernel of the Jacobian.
     However, in general, this implication is not an equivalence (i.e., critical points of algebraic maps can have fibers of the expected dimension).
\end{rmk}

\subsection{Relation to SLT}\label{sec:relslt}
Here, we draw a parallel between critical points and SLT. We highlight that the traditional formalism of SLT does not fully coincide with ours. Still, several intersections occur, some of which can be made completely formal, while others are less rigorous. 

As mentioned in Section \ref{sec:relwork}, SLT focuses on parameters where the Fisher Information Metric degenerates. To formalize this, assume that the ambient space $\mathcal{V}$ is equipped with a Riemannian metric, i.e., a smoothly-varying symmetric positive-definite matrix $G_f$ for $f \in \mathcal{V}$. This metric can be pulled back to the parameter space as $(\varphi^*G)_\mathbf{W} = (\textnormal{J}_\mathbf{W}\varphi)^\top G_{\varphi(\mathbf{W})} (\textnormal{J}_\mathbf{W}\varphi)$. However, the pulled-back matrix $(\varphi^*G)_\mathbf{W}$ might be degenerate. In this case, $\mathbf{W}$ is deemed \emph{singular}\footnote{We stick to the concise definition from, e.g., \cite{watanabe2007almost}. Other sources loosen the notion of singularity via identifiably conditions.}. Since $\textnormal{rk}\;(\varphi^*G)_\mathbf{W} = \textnormal{rk}\; \textnormal{J}_\mathbf{W}\varphi$, this happens precisely when the rank of the Jacobian is smaller than the number of parameters $\dim \mathcal{W}$. As discussed above, for almost all $\mathbf{W}$, we have $\textnormal{rk}\; \textnormal{J}_\mathbf{W}\varphi = \textnormal{dim}\; \mathcal{M}_{\mathbf{d}, r}$, while at critical points the Jacobian's rank drops. This means that, while all parameters are singular, critical points are the `most singular ones'.      

According to SLT, singularities play a fundamental role in  machine learning, representing an implicit bias of the learning process. More specifically, a fundamental quantity in SLT is the \emph{Local Learning Coefficient} (LLC) \cite{watanabe2009algebraic, lau2023local} -- also referred to as (local) real log canonical threshold. In the probabilistic formalism, the LLC controls the objective optimized by the model (i.e., the negative log-likelihood of data) -- this is the well-known Bayesian Information Criterion \cite{watanabe2013widely}. As a consequence, the model is biased towards parameters with low LLC. Now, the LLC is closely related to the rank of the Jacobian of the parametrization. In fact, assuming that the quadratic approximation around $\mathbf{W}$ of the Kullback-Leibler divergence between the model and the data distribution is exact, the LLC coincides with $\frac{1}{2} \textnormal{rk}\; \textnormal{J}_\mathbf{W}\varphi $ \cite[Section 4]{wei2022deep}. In conclusion, all this can be interpreted as establishing an implicit bias towards critical points, measured by the rank of their Jacobian.

\section{Main Result}
\label{sec:mainres}
In this section, we state and prove our main result, concerning the equivalence between critical points and subnetworks. Denote by $ D= \max_{1 \leq t < L}d_t$ the maximum width of the network. 
\begin{thm}
\label{thm:maincrit}
Suppose that $r > 9(D(D-1))^{L-1}$. Then, $\mathbf{W} \in \mathcal{W}$ is a subnetwork if and only if it is critical for the parametrization map $\varphi$.  
\end{thm}

One direction of this assertion follows immediately from the basic observations already discussed so far: 
If $\mathbf{W}$ is a subnetwork, then the fiber $\varphi^{-1}(f_{\mathbf{W}})$ is larger than the generic one (see end of Section \ref{sec:backident}) and so $\mathbf{W}$ is critical by Remark \ref{rmk:largeFiberImpliesCritical}. Hence, in what follows, it is sufficient to consider non-subnetworks and prove that they are regular points of $\varphi$.

\begin{rmk}\label{rmk:noWidth1}
From now on, we may assume that $d_i > 1$ for $1 \leq i < L$, i.e., that the hidden layers of the network contain at least two neurons. Indeed, when $d_i = 1$ for some $i$, the second condition of Definition \ref{def:subnet} implies that all parameters are subnetworks, unless $d_j = 1$ for all $1 \leq j < L$. In the latter case, it is straightforward to check that the results from this section hold without any assumption on $r$.      
\end{rmk}

Throughout this section, we will work with complex scalars, i.e., we will assume that the parameters, inputs, and weights take values in $\mathbb{C}$. This is not only convenient from an algebraic perspective, but actually leads to stronger results. Indeed, if Theorem \ref{thm:maincrit} holds over $\mathbb{C}$, then it holds over $\mathbb{R}$ as well, since all the involved notions (subnetworks, Jacobian, regularity, etc.) are compatible.

Before proving Theorem \ref{thm:maincrit}, we establish a number of tools and preliminary results. Recall that Mason's Theorem for triplets of polynomials is a function-field analog of the number-theoretical $abc$-conjecture. We will use a generalized form of Mason's theorem from \cite{de2007another}, concerning tuples of multivariate polynomials.
\begin{thm}[{\cite[Theorem 1.3]{de2007another}}]
\label{thm:genmason}
Let $U_1, \dots, U_n$, $n\geq 2$, be multivariate complex polynomials that are not all constant. Suppose that $\gcd(U_1, \ldots, U_n) = 1$, and that $\sum_{i=1}^{n} U_i = 0$ is a minimal vanishing sum, i.e., no proper subsum is zero. 
Then: 
\begin{equation}
\max_i \deg(U_i) \leq (n-2)\left( \sum_{i=1}^n r_i -1 \right),
\end{equation}
where $r_i$ is the degree of the radical of $U_i$, i.e., the sum of the degrees of its prime factors. 
\end{thm}
The above theorem also applies without the coprimality assumption on the $U_i$. In that case, the inequality becomes 
\begin{equation}
\label{eq:ineqrearr}
    \max_i \deg(U_i) - \deg(\gcd(U_1, \ldots, U_n)) \leq (n-2) \left( \sum_{i=1}^n r_i -1 \right), 
\end{equation}
where $r_i$ is the degree of the radical of $U_i / \gcd(U_1, \ldots, U_n)$. The latter is bounded by the degree of the radical of $U_i$. 

We now exploit Mason's Theorem to prove our main result. As anticipated in Section \ref{sec:intro}, Mason's Theorem is stronger than Newman-Slater's Theorem, which lies at the heart of the identifiability argument by \cite{finkel2024activation} -- see Section \ref{sec:app_ident}. In fact, Mason's theorem was recently used in \cite{cruaciun2025linear} to give an alternative proof of identifiability. 


\begin{lemm}
\label{lemm:shareddeg}
Suppose that $\mathbf{W} \in \mathcal{W}$ is not a subnetwork, and that $r > 4D(D -1)$. For distinct $i,j =1, \ldots, d_{L} $, denote by $g_i$ and $g_j$ the quotients of $f_{\mathbf{W}}[i]$ and $f_{\mathbf{W}}[j]$ by their gcd, respectively. If $f_{\mathbf{W}}[i]$ and $f_{\mathbf{W}}[j]$ are not proportional\footnote{This can happen when some rows of $W_L$ are proportional, which is allowed by the definition of subnetwork.}, then:
\begin{equation}  
\label{eq:ineqfinal}
\max \{ \deg(g_i), \deg(g_j)\} \geq \left( \frac{r- D(D-1)}{ D(D - 1)} \right)^{L-1}. 
\end{equation}
\end{lemm}
\begin{proof}
We proceed by induction on the depth $L$. For $L=1$ (i.e., for a single linear layer), since 
$f_{\mathbf{W}}[i]$ and $f_{\mathbf{W}}[j]$ are non-proportional linear forms, they are coprime. Thus, $\deg(g_i) = \deg(g_j) = 1$, as desired.

Suppose that $L > 1$. By definition, $f_{\mathbf{W}}[i] = \sum_k W_L[i,k] \ p_k^r$, where $p_k$ is the output of the $k$-th neuron of the penultimate layer. We have:
\begin{equation}   
\begin{aligned}
\sum_{k=1}^{d_{L-1}} \underbrace{(W_L[j,k] g_i - W_L[i,k] g_j)}_{C_k} \ p_k^r 
&= g_i f_{\mathbf{W}}[j] - g_j f_{\mathbf{W}}[i] \\
&= g_i \gcd(f_{\mathbf{W}}[i],f_{\mathbf{W}}[j]) g_j - g_j \gcd(f_{\mathbf{W}}[i],f_{\mathbf{W}}[j]) g_i \\ &=0 .
\end{aligned}
\end{equation}
Now, consider a minimal vanishing subsum of this sum, given by an index set $S \subseteq \{1, \ldots, d_{L-1} \}$, with $|S| > 1$. This is possible, since not all the terms in the sum are vanishing. Indeed, if all $C_kp_k^r$ vanished, then, since $p_k \neq 0$ for all $k$ by Corollary \ref{cor:penultimate}, we would have $C_k=0$ for every $k$. Thus, $W_L[j,k]g_i = W_L[i,k]g_j$ for every $k$. Since $g_i$ and $g_j$ are obtained from the non-proportional polynomials $f_{\mathbf{W}}[i]$ and $f_{\mathbf{W}}[j]$ by dividing by their common gcd, they are non-proportional as well. Hence $W_L[j,k]=W_L[i,k]=0$ for every $k$, and thus
$f_{\mathbf{W}}[i]$ and $f_{\mathbf{W}}[j]$ would be proportional, contrary to the assumption. 

Next, for $k \in S$, we define $\hat{p}_k$ as the quotient of $p_k$ by $\gcd(p_s \colon s \in S)$. Since $|S| > 1$ and the $p_k$ are pairwise non-proportional by Corollary \ref{cor:penultimate}, the inductive hypothesis implies that 
\begin{equation}
\label{eq:indineq}
\max_{k \in S} \deg(\hat{p}_k) \geq  \left( \frac{r-D'(D'-1)}{ D'(D' - 1)} \right)^{L-2}, 
\end{equation}
where $D' := \max_{1 \leq t < L-1} d_t$.

Setting $U_k := C_k \hat{p}_k^r$,  we wish to apply Theorem \ref{thm:genmason} to the minimal vanishing sum $\sum_{k \in S} U_k = 0$. Since the $\hat{p}_k$ are coprime for $k\in S$, $\gcd(U_k \colon k \in S)$ divides $\prod_{k \in S}C_k$. In particular:
\begin{equation}
\deg(\gcd(U_k \colon k \in S)) \leq \sum_{k\in S} \deg(C_k) \leq |S|G,
\end{equation}
where $G = \max \{ \deg(g_i), \deg(g_j)\}$. Thus, the left-hand side of \eqref{eq:ineqrearr} is lower-bounded by $r \max_k \deg(\hat{p}_k) - |S| G$. 
Moreover, note that the radical of $U_k$ divides $C_k \hat{p}_k$, and its degree is at most $G +  \deg(\hat{p}_k)$. Thus, the right-hand side of \eqref{eq:ineqrearr} is upper-bounded by $|S|(|S| - 2)(G + \max_{k\in S} \deg(\hat{p}_k))$. All in all, we obtain an inequality:
\begin{equation}
  r \max_{k \in S} \deg(\hat{p}_k) - |S|G \leq  |S|(|S| - 2)(G + \max_{k\in S} \deg(\hat{p}_k)).  
\end{equation}
By rearranging, since $|S| \leq d_{L-1}$, we get: 
\begin{equation}
\begin{aligned}
   G &\geq \frac{(r - |S|(|S|- 2)) \max_{k\in S}\deg(\hat{p}_k) }{|S|(|S| - 1)} \\
   &\geq \frac{r - d_{L-1}(d_{L-1} - 1)}{d_{L-1}(d_{L-1} - 1)} \max_{k\in S}\deg(\hat{p}_k) \\
   & \geq \frac{r - d_{L-1}(d_{L-1} - 1)}{d_{L-1}(d_{L-1} - 1)} \left( \frac{r-D'(D' - 1)}{ D'(D' - 1)} \right)^{L-2} \\
   & \geq \left( \frac{r-D(D - 1)}{ D(D - 1)} \right)^{L-1}.  
\end{aligned}
\end{equation}
\end{proof}


\begin{lemm}\label{lem:linear-dependence-multi}
Let $L > 1$.
 Suppose that $\mathbf{W}$ is not a subnetwork. Denote by $p_1, \ldots, p_{d_{L-1}}$ the output polynomials of the penultimate layer of the network, and let $h_1, \ldots, h_{d_{L-1}}$ be nonzero complex polynomials (in $d_0$ variables) of degree at most $r^{L-2}$ such that $h_i$ and $p_i$ are non-proportional for every $i$.  
  If the set $T := \{p_i^r \}_{i} \cup \{p_i^{r-1}h_i \}_{i }$ is linearly dependent, then 
 \begin{equation}
   \label{eq:boundlowm}
   r \leq 9(D(D-1))^{L-1}. 
 \end{equation}

\end{lemm}
\begin{proof}
If $r \leq 4D(D-1)$, then \eqref{eq:boundlowm} holds trivially. Hence, from now on, we assume that $r > 4D(D-1)$ (which will allow us to apply Lemma \ref{lemm:shareddeg}).

Suppose that $T$ is linearly dependent. Since the coefficients of the linear combination can be absorbed by $\mathbf{W}$, we can assume, without loss of generality, that $\sum_{U \in T}U = 0$. Similarly to the proof of Lemma \ref{lemm:shareddeg}, we consider a minimal subsum of this sum and then apply Theorem \ref{thm:genmason}. Note that this subsum cannot contain only $p_i^r$ and $p_i^{r -1}h_i$ for some $i$, since otherwise $p_i$ and $h_i$ would be proportional. But then, since the $p_i$ are pairwise non-proportional by Corollary \ref{cor:penultimate}, we can apply Lemma \ref{lemm:shareddeg} to the network with the last layer removed. The left-hand side of \eqref{eq:ineqrearr} is lower-bounded by $r-1$ times the right-hand side of \eqref{eq:ineqfinal} (with $L$ replaced by $L-1$). 
Since $|T| = 2d_{L-1}$, the right-hand side of \eqref{eq:ineqrearr} is instead upper-bounded by $|T|(|T| - 2) \max_{i}( \deg(p_i ) +  \deg(h_i))  \leq  4D(2D - 2) r^{L-2}$. We conclude that 
\begin{equation}
\label{eq:ineqbern}
    (r-1)\left( \frac{r - D(D-1)}{ D(D - 1)} \right)^{L-2}  \leq 8D(D - 1) r^{L-2}. 
\end{equation}
Multiplying both sides by $(\frac{D(D-1)}{r})^{L-2}$ yields
\begin{equation}
    (r-1)\left( 1 - \frac{D(D-1)}{r
} \right)^{L-2}  \leq 8(D(D - 1))^{L-1}.
\end{equation}
Since $r > D(D-1)$, we can apply Bernoulli's inequality $(1 - x)^\alpha \geq 1- \alpha x$ for $x \leq 1$, obtaining 
\begin{equation} \label{eq:almostThere}
  r - 1 - (L-2)D(D-1) \leq  (r-1)\left(1 - \frac{(L-2)D(D-1)}{r}\right) \leq  8(D(D - 1))^{L-1}.
\end{equation}
By Remark \ref{rmk:noWidth1}, we assume $D > 1$, and so $1 + (L-2)D(D-1) \leq (D(D-1))^{L-1}$ holds. Thus, \eqref{eq:almostThere} implies  \eqref{eq:boundlowm}.
\end{proof}

\begin{prop}
\label{thm:main_multilayer}
Suppose that $\mathbf{W}$ is not a subnetwork, and that $r > 9(D(D-1))^{L-1}$. Then the kernel of the Jacobian $\textnormal{J}_\mathbf{W}\varphi $ is spanned by weights of the form $(0, \dots, 0, \Lambda W_{i-1}, -r W_i \Lambda, 0, \dots, 0) $, where $2 \leq i \leq L$ and $\Lambda$ is a diagonal matrix. 
\end{prop}
\begin{proof}
We proceed by induction on the depth $L$. For $L=1$ (i.e., for a single linear layer), the Jacobian has trivial kernel, and the statement is trivial.  

Suppose that $L > 1$. The kernel of the Jacobian is given by the solutions in $\mathbf{M}$ of:
\begin{equation} \label{eq:induction_eq}
\textnormal{J}_\mathbf{W}\varphi (\mathbf{M}) = M_L p^r +  r W_L  (p^{r-1} \odot h) = 0,
\end{equation}
where $p=(p_1, \ldots, p_{d_{L-1}})$ are the outputs of the penultimate layer, $\odot$ is the Hadamard product, and $h = (h_1, \ldots, h_{d_{L-1}})$ are the outputs of the Jacobian of the MLP with the last layer removed, in the direction $(M_1, \ldots, M_{L-1})$. We distinguish two cases. 

\textbf{Case 1: $M_L = 0$.} For every $i = 1, \ldots, d_{L-1}$, since $\mathbf{W}$ is not a subnetwork, $W_L[j,i] \not = 0$ for some $j = 1, \ldots, d_L$. But then the $j$-th component of \eqref{eq:induction_eq} defines a linear combination, where $p_i^{r-1}h_i$ appears with a nonvanishing coefficient. Lemma \ref{lem:linear-dependence-multi} implies that each $h_i$ is either zero or proportional to $p_i$. Thus, we can write $h_i=\lambda_i p_i$ for some scalars $\lambda_i\in \mathbb{C}$. Substituting this into \eqref{eq:induction_eq}, and using $M_L=0$, gives $W_L\Lambda p^r=0$, where $\Lambda=\operatorname{diag}(\lambda_i)$. Since the polynomials
$p_i^r$ are linearly independent by Theorem \ref{thm:newslat}, and no column of $W_L$ vanishes because $\mathbf{W}$ is not a subnetwork, we obtain $\lambda_i=0$ for all $i$ and hence $h_i=0$ for all $i$. Thus, $(M_1, \ldots, M_{L-1})$ lies in the kernel of the Jacobian of the truncated network, and we conclude by the inductive hypothesis. 

\textbf{Case 2: $M_L \ne 0$.}
Similarly, for every $i,j$ such that $M_L[j,i]\not = 0$, Lemma \ref{lem:linear-dependence-multi} implies that $h_i = \lambda_i p_i$ for some scalar $\lambda_i \in \mathbb{C}$ (independent of $j$), and $M_L[j,i] p_i^r + r \lambda_i  W_L[j,i] p_i^r = 0 $.
Since $p_i \neq 0$ by Corollary \ref{cor:penultimate}, we obtain $M_L[j,i]  + r \lambda_i  W_L[j,i]  = 0 $.
Letting $\Lambda$ be the diagonal matrix with diagonal entries $\lambda_i$ (if $M_L[:,i] = 0$, simply set $\lambda_i = 0$),  we get that $M_L = -r W_L \Lambda$. Note that $p$ itself lies in the image of the Jacobian of the network parametrization with the last layer removed
(specifically, it is the image of $(0, \ldots, 0, W_{L-1})$ under that Jacobian).
Thus, the identity $h - \Lambda p = 0$ shows that $(M_1,\ldots,M_{L-2}, M_{L-1}-\Lambda W_{L-1})$
lies in the kernel of the Jacobian of the truncated network. 
Since $M = (M_1,\ldots,M_{L-2}, M_{L-1}-\Lambda W_{L-1},0) + (0, \ldots, 0, \Lambda W_{L-1}, -r W_L \Lambda)$, the desired statement follows then from the inductive hypothesis. 
\end{proof}


Finally, we prove Theorem \ref{thm:maincrit}. 
\begin{proof}[Proof of Theorem \ref{thm:maincrit}]
We have argued already, directly after the statement of Theorem \ref{thm:maincrit}, that subnetworks are critical points. For the converse direction, assume that $\mathbf{W}$ is not a subnetwork. Proposition \ref{thm:main_multilayer} implies that the kernel of the Jacobian $\textnormal{J}_\mathbf{W} \varphi$ has dimension $\sum_{i=1}^{L-1}d_i$, due to the degrees of freedom in choosing the diagonal matrices $\Lambda$. By Theorem \ref{thm:finkel}, this dimension coincides with the one of the generic fiber of $\varphi$, implying that $\mathbf{W}$ is a regular point for $\varphi$ by the fiber-dimension theorem.  
\end{proof}

\section{Conclusions and Future Work}
In this work, we have considered deep neural networks with a monomial activation function and, via tools from polynomial algebra, established an equivalence (for large activation degree)
between critical points of the parametrization and subnetworks. Since critical points are known to be related to implicit biases (e.g., via SLT), this offers a mathematical justification for the (implicit) Occam's Razor of deep neural networks. 

A core natural question for future investigation is whether the relation between subnetworks and criticality extends to other scenarios.
It would be particularly interesting to show this result for generic polynomial activations of sufficiently large degree. Since polynomials can approximate arbitrary activation functions (by the Stone--Weierstra\ss{} Theorem), a meta-belief of Neuroalgebraic Geometry is that the behavior of networks with generic polynomial activations is representative for the behavior of networks with sufficiently general non-polynomial activations. 
Techniques similar to \cite{shahverdi2026learning} can potentially be applied to polynomial non-monomial networks. 
For specific other activations of interest, such as ReLU, the network might behave differently than in the (generic) polynomial case. In fact, similar problems are known to be extremely challenging in this case \cite{grigsby2023hidden, grillo2026most, bona2023parameter}, and a clean equivalence between criticality and subnetworks might not hold. 

On a different note, architectures other than fully-connected networks should be explored. 
Whether subnetworks are equivalent to critical parameters or not, depends highly on the chosen architecture. For instance, for convolutional networks with either monomial or generic polynomial activation, the network parametrization has no critical points (except for parameters that produce the zero-function) \cite{shahverdi2024geometryoptimizationpolynomialconvolutional,shahverdi2026learning}.
Important examples where critical points have not been fully classified are ResNets and attention-based networks. The latter are pivotal in a vast array of applications, and have been previously studied, in their un-normalized form, via algebraic tools \cite{henry2024geometrylightningselfattentionidentifiability}.  


\section*{Acknowledgements}
GLM, VS, and KK were  supported by the Wallenberg AI, Autonomous Systems and Software Program (WASP) funded by the Knut and Alice Wallenberg Foundation. WW was supported by the ERC Consolidator Grant FourSurf 101087365.

\bibliography{main}
\bibliographystyle{alpha}

\newpage

\appendix

\section{Identifiability of Monomial Networks}
\label{sec:app_ident}
Here, we state a stronger version of the main result of \citep{finkel2024activation}, with a similar proof. Similarly to Section \ref{sec:mainres}, we will work with complex scalars, resulting in a stronger statement than over $\mathbb{R}$. We will need the following general result from polynomial algebra. 

\begin{thm}[\cite{newman1979waring, finkel2024activation}]
\label{thm:newslat}
Let $U_1, \dots, U_n$, $n\geq 2$ be multivariate complex polynomials that are non-vanishing and pairwise non-proportional. If $r > n(n-2)$, then $U_1^r, \ldots, U_n^r$ are linearly independent. 
\end{thm}

The above result can be deduced, straightforwardly, from Theorem \ref{thm:genmason} -- this is the same strategy as in \cite{cruaciun2025linear}. When $n=3$, the above result can be interpreted as a polynomial version of Fermat's Last Theorem.

\begin{thm}
\label{thm:finkel}
Suppose that $r > 4D(D -1)$. Let $\mathbf{W}, \mathbf{V} \in \mathcal{W}$, and assume that $\mathbf{W}$ is not a subnetwork. If $f_\mathbf{W}(x) = f_{\mathbf{V}}(x) $ for all $x \in \mathbb{R}^{d_0}$, then 
\begin{equation}
\label{eq:fiberdes}
\begin{aligned}
    V_1 &= P_1D_1W_1 \\
    V_i &= P_i D_iW_i D_{i-1}^{- r} P_{i-1}^\top \hspace{5em} i=2, \ldots, L-1 \\
    V_L &= W_L D_{L-1}^{-r} P_{L-1}^\top,
\end{aligned}
\end{equation}
where $P_i$ is a $d_{i} \times d_{i}$ permutation matrix, and $D_i$ is a diagonal invertible $d_{i} \times d_{i}$ matrix. 
\end{thm}

\begin{proof}
We will proceed by induction on the depth $L$. For $L = 1$, the statement is straightforward. We will thus assume $L> 1$. 

 Denote by $p_j$ the output of the $j$-th neuron in the penultimate layer of $f_\mathbf{W}$. We claim that the $p_j$ are non-vanishing. This can be deduced from the inductive hypothesis, since $p_j$ is an MLP with architecture $(d_0, \ldots, d_{L-2}, 1)$ and weights $(W_1, \ldots, W_{L-2}, W_{L-1}[j, :])$. A subtlety is that, even though $\mathbf{W}$ is not a subnetwork, the weights of $p_j$ might be, since some entries of $W_{L-1}[j,:]$ can vanish (but not all of them), resulting in inactive neurons at the $(L-2)$-th layer. However, as explained in Section \ref{sec:backident}, these weights can be reduced. Specifically, we remove the vanishing entries from $W_{L-1}[j,:]$, together with the corresponding rows of $W_{L-2}$. This operation can, potentially, generate some (but not all) redundant neurons with vanishing outcoming weights at the $(L-3)$-th layer\footnote{Note that this operation can not generate redundant neurons, nor neurons with vanishing incoming weights.}. By iteratively removing such redundant neurons at early layers, we obtain an MLP with a smaller architecture and non-subnetwork weights, parametrizing the polynomial $p_j$ and with $W_{L-1}[j,:] \neq  0$. If $p_j$ were 0, it could be  realized both by the non-subnetwork weights we just constructed and by the all vanishing weights ($f_\mathbf{0} = 0$), which contradicts the inductive hypothesis.
 
We now claim that the $p_j$ are pairwise non-proportional. Given $j \neq k$, suppose that $p_j = \alpha \ p_k$ for some  $\alpha \in \mathbb{C}$ and consider the MLP with architecture $(d_0, \ldots, d_{L-2}, 1)$ and weights $(W_1, \ldots, W_{L-2}, W_{L-1}[j, :] - \alpha \ W_{L-1}[k, :])$, which parametrizes the zero polynomial. By reasoning as above, due to $W_{L-1}[j, :] \neq \alpha \  W_{L-1}[k, :]$, the induction hypothesis results, again, in a contradiction.

Suppose now that $f_\mathbf{W} = f_{\mathbf{V}}$. Then for all $i = 1, \ldots, d_L$, we have: 
\begin{equation}
\label{eq:mlpident}
    \sum_{j=1}^{d_{L-1}} W_L[i,j] \ p_j^r -  \sum_{k=1}^{d_{L-1}} V_L[i,k] \  q_k^r = 0,  
\end{equation}
where $q_k$ is the output of the $k$-th neuron in the penultimate layer of $f_\mathbf{V}$. Since $W_L$ has no vanishing columns, for every $j = 1, \ldots, d_{L-1}$  there exists $i = 1, \ldots, d_L$ such that $W_L[i, j] \not = 0$. For any such $i$, we apply Theorem \ref{thm:newslat} (with $n \leq 2d_{L-1} \leq 2D$) to the vanishing linear combination \eqref{eq:mlpident}, after potentially removing vanishing terms and grouping proportional ones in the sum indexed by $k$. We conclude that there exists $k = 1, \ldots, d_{L-1}$ such that $  p_j = \lambda_j \ q_k$ for some $\lambda_j \in \mathbb{C}$. Since the $p_j$ are pairwise non-proportional and non-zero, two different indices $j$ cannot be associated with the same $k$. Thus, the map $j\mapsto k$ is injective, and
hence bijective since both index sets have cardinality $d_{L-1}$. We set $P_{L-1}$ as the corresponding permutation matrix, and $D_{L-1}$ as the diagonal matrix whose $j$-th entry is $\lambda_j^{-1}$.
That way, we obtain $q = P_{L-1}D_{L-1}p$.
Moreover, since the $p_j^r$ (resp. the $q_k^r$) are linearly independent by Theorem \ref{thm:finkel}, the coefficients in the linear combination $\sum_j W_L[i,j] p_j^r$ (resp. $\sum_k V_L[i,k] q_k^r$) are uniquely determined, yielding $V_L = W_L D_{L-1}^{-r} P_{L-1}^\top$.
Now we apply the induction hypothesis to the architecture $(d_0, \ldots, d_{L-1})$ with  weights $\mathbf{W'} := (W_1, \ldots, W_{L-2}, P_{L-1}D_{L-1}W_{L-1})$ and $\mathbf{V'} := (V_1, \ldots, V_{L-1})$, which give rise to the same function due to $q = P_{L-1}D_{L-1}p$. 
That way, we obtain permutation and diagonal matrices $P_t$ and $D_t$ for $t = 1, \ldots, L-2$ such that \eqref{eq:fiberdes} holds, as desired. 
\end{proof}

\begin{cor}\label{cor:penultimate}
    Suppose that $r > 4D(D-1)$ and that $\mathbf{W}$ is not a subnetwork. The output polynomials of the penultimate layer of $f_{\mathbf{W}}$ are non-zero and pairwise non-proportional.
\end{cor}
\begin{proof}
    This is proven inside the above proof of Theorem \ref{thm:finkel}.
\end{proof}

\end{document}